\documentclass{article}

\PassOptionsToPackage{numbers, compress}{natbib}
\usepackage[final]{nips_2016}

\usepackage[utf8]{inputenc} 
\usepackage[T1]{fontenc}    
\usepackage{url}            
\usepackage{booktabs}       
\usepackage{amsfonts}       
\usepackage{nicefrac}       
\usepackage{microtype}      
\usepackage{amsmath,epsfig}
\usepackage{amssymb}
\usepackage{amsthm}
\usepackage{algorithm}
\usepackage{algpseudocode}
\usepackage{tikz}
\usepackage{tabu}
\usepackage{graphicx}
\usepackage{subfig}
\usepackage{epstopdf}
\usepackage{epsfig}
\usepackage{caption}
\usepackage{float}
\usepackage{multirow}
\numberwithin{equation}{section} 
\newtheorem{theorem}{Theorem}[section]                   

\newtheorem{definition}[theorem]{Definition}

\usepackage{microtype}
\usepackage[small,compact]{titlesec}

\newcommand{\new}[1]{{#1}}

\title{Iterative Thresholding for Demixing \\ Structured Superpositions in High Dimensions}

\author{
  Mohammadreza Soltani \\
  Iowa State University\\
 \And
 Chinmay Hegde \\
 Iowa State University
}

\begin{document}

\maketitle

\begin{abstract}
We consider the demixing problem of two (or more) high-dimensional vectors from nonlinear observations when the number of such observations is far less than the ambient dimension of the underlying  vectors. Specifically, we demonstrate an algorithm that stably estimate the underlying components under general \emph{structured sparsity} assumptions on these components. Specifically, we show that for certain types of structured superposition models, our method provably recovers the components given merely $n = \mathcal{O}(s)$ samples where $s$ denotes the number of nonzero entries in the underlying components. Moreover, our method achieves a fast (linear) convergence rate, and also exhibits fast (near-linear) per-iteration complexity for certain types of structured models. We also provide a range of simulations to illustrate the
performance of the proposed algorithm.
\end{abstract}

\section{Introduction}
\label{Intro}
The \emph{demixing} problem involves disentangling two (or more) high-dimensional vectors from their linear superposition~\cite{mccoyTropp2014,mccoy2014convexity,soltani2016fast,SoltaniHegde_Asilomar,SoltaniHegde_Globalsip}. In statistical learning applications involving parameter estimation, such superpositions can be used to model situations when there is some ambiguity in the parameters (e.g., the true parameters can be treated as ``ground truth'' + ``outliers'') or when there is some existing prior knowledge that the true parameter vector is a superposition of two components. Mathematically, suppose that the parameter vector is given by $\beta =\Phi\theta_1 + \Psi \theta_2$ where $\beta,\theta_1, \theta_2\in\mathbb{R}^p$ and $\Phi, \Psi$ are orthonormal bases. 
If a linear observation model is assumed, then given samples $y\in\mathbb{R}^n$ and a design matrix $X \in \mathbb{R}^{n \times p}$, the goal is to recover the parameter vector $\beta$ that minimizes a loss function $\mathcal{L}(X,y; \beta)$. We focus on the sample-poor regime where the dimension far exceeds the number of samples; this regime has received significant attention from the machine learning and signal processing communities in recent years~\cite{negahban2009unified,CandesCS}.  

However, fitting the observations according to a linear model can be restrictive. One way to ease this restriction is to assume a \emph{nonlinear} observation model:
\begin{align}\label{nonlindex}
 y = g(X\beta) + e =g(X(\Phi\theta_1 + \Psi \theta_2) )+ e,
 \end{align}
where $g$ denotes a nonlinear \textit{link} function and $e$ denotes observation noise. This is akin to the \emph{Generalized Linear Model} (GLM) and \emph{Single Index Model} (SIM) commonly used in statistics~\cite{kakade2011}. Here, the problem is to estimate $w$ and $z$ from the observations $y$ with as few samples as possible.

The above estimation problem is challenging in several different aspects: (i) there is a basic identifiability of issue of obtaining $\theta_1$ and $\theta_2$ even with perfect knowledge of $\beta$; (ii) there is a second identifiability issue arising from the nontrivial null-space of the design matrix (since $n \ll p$); and (iii) the nonlinear nature of $g$, as well as the presence of noise $e$ can further confound recovery.

Standard techniques to overcome each of these challenges are well-known. By and large, these techniques all make some type of \emph{sparseness} assumption on the components $\theta_1$ and $\theta_2$ \cite{CandesCS}; some type of \emph{incoherence} assumption on the bases $\Phi$ and $\Psi$ \cite{elad2005simultaneous,donoho2006stable}; some type of \emph{restricted strong convexity} (RSC) 
~\cite{negahban2009unified}; and some type of \emph{Lipschitz} (\emph{restricted strong smoothness} (RSS)) assumptions on the link function $g$~\cite{yang2015sparse}. See section \ref{Perm} for details.

In this short paper, we demonstrate an algorithm that stably estimate the components $\theta_1$ and $\theta_2$ under general \emph{structured sparsity} assumptions on these components. Structured sparsity assumptions are useful in applications where the support patterns (i.e., the coordinates of the nonzero entries) belong to certain restricted families (for example, the support is assumed to be \emph{group-sparse} \cite{huang2010benefit}). It is known that such assumptions can significantly reduce the required number of samples for estimating the parameter vectors, compared to generic sparsity assumptions~\cite{modelcs,SPINIT,approxIT}. 

We note that demixing approaches in high dimensions with structured sparsity assumptions have appeared before in the literature~\cite{mccoyTropp2014,mccoy2014convexity,rao2014forward}. However, our method differs from these earlier works in a few different aspects. The majority of these methods involve solving a convex relaxation problem; in contrast, our algorithm is manifestly \emph{non-convex}. Despite this feature, for certain types of structured superposition models our method provably recovers the components given merely $n = \mathcal{O}(s)$ samples; moreover, our methods achieve a fast (linear) convergence rate, and also exhibits fast (near-linear) per-iteration complexity for certain types of structured models. Moreover, these earlier methods have not explicitly addressed the nonlinear observation model (with the exception of \cite{plan2014high}). We show that under certain smoothness assumptions on $g$, the performance of our method matches (in terms of asymptotics) the best possible sample-complexity.

\section{Preliminaries}
\label{Perm}
Let $\|.\|_q$ denote the $\ell_q$-norm of a vector. Denote the spectral norm of the matrix $X$ as $\|X\|$. Denote the true parameter vector, $\theta =  [\theta_1^T \ \theta_2^T ]^T\in\mathbb{R}^{2p}$ as the vector obtaining by stacking the true and unknown coefficient vectors,  $\theta_1, \theta_2$. For simplicity of exposition, we suppose that components $\theta_1$ and $\theta_2$ have block sparsity with sparsity $s$ and block size $b$~\cite{modelcs} (Analogous approaches apply for other structured sparsity models.)

The problem~\eqref{nonlindex} is inherently unidentifiable and to resolve this issue, we need to assume that the coefficient vectors $\theta_1, \theta_2$ are distinguishable from each other. This issue is characterized by a notion of incoherence of the components $\theta_1, \theta_2$~\cite{SoltaniHegde_Globalsip}.

\begin{definition}\label{incoherence}
The bases $\Phi$ and $\Psi$ are called $\varepsilon$-incoherent if 
$
\varepsilon = \sup_{\substack{\|u\|_0\leq s,\ \|v\|_0\leq s  \\ \|u\|_2 = 1,\ \|v\|_2 = 1}}|\langle{\Phi u, \Psi v}\rangle|.
$
\end{definition}
For the analysis of our proposed algorithm we need the following standard definition~\cite{negahban2009unified}:
\begin{definition}\label{rssrsc}
$f : \mathbb{R}^{2p} \rightarrow \mathbb{R}$ satisfies \new{Structured} \textit{Restricted Strong Convexity/Smoothness \new{(SRSC/SRSS)} }if:
\begin{align*}
m_{4s}\leq\|\nabla^2_{\xi} f(t)\|\leq M_{4s},\  \ t\in\mathbb{R}^{2p},
\end{align*}
where $\xi = \textrm{supp}(t_1)\cup \textrm{supp}(t_2)$, for all $t_i\in\mathbb{R}^{2p}$ such that \new{$t_i$ belongs to $(2s,b)$ block-sparse vectors} 
for $i=1,2$, and $m_{4s}$ and $M_{4s}$ are (respectively) the \new{SRSC and SRSS} constants. Also $\nabla^2_{\xi} f(t)$ denotes a $4s\times 4s$ sub-matrix of the Hessian matrix $\nabla^2 f(t)$ comprised of row/column indices indexed by $\xi$.
\end{definition}

Also, we assume that the derivative of the link function is strictly bounded either within a positive interval, or within a negative interval.

\section{Algorithm and main theory}
In this section, we describe our algorithm which we call it \emph{Structured Demixing with Hard Thresholding} (STRUCT-DHT) and our main theory. To solve demixing problem in~\eqref{nonlindex}, we consider the minimization of a special loss function $F(t)$ following~\cite{SoltaniHegde_Globalsip}:
\begin{equation} \label{optprob}
\begin{aligned} \underset{t \in \mathbb{R}^{2p}}{\text{min}}
\ \ F(t) &= \frac{1}{m}\sum_{i=1}^m \Theta(x_i^T\Gamma t) - y_i x_i^T\Gamma t \\
& \text{s.\ t.} \  \  t\in\mathcal{D}
\end{aligned}
\end{equation}
where $\Theta(x) = \int_{-\infty}^{x} g(u)du$ denotes as the integral of the link function $g$, $\Gamma = [\Phi \  \Psi]$, $x_i$ is the $i^{th}$ row of the design matrix $X$ and $\mathcal{D}$ denotes the set of length-$2p$ vectors formed by stacking a pair of $(s,b)$ block-sparse vectors. The objective function in~\eqref{optprob} is motivated by the single index model in statistics; for details, see~\cite{SoltaniHegde_Globalsip}. To approximately solve~\eqref{optprob}, we propose STRUCT-DHT which is detailed as Algorithm \ref{algHTM}.

\begin{algorithm}[t]
\caption{Structured Demixing with Hard Thresholding (STRUCT-DHT)
\label{algHTM}
}
\begin{algorithmic}
\State \textbf{Inputs:} Bases $\Phi$ and $\Psi$, design matrix $X$, link function $g$, observation $y$, sparsity $s$, step size $\eta'$. 
\State \textbf{Outputs:} Estimates  $\widehat{\beta}=\Phi\widehat{\theta_1} + \Psi\widehat{\theta_2}$, $\widehat{\theta_1}$, $\widehat{\theta_2}$
\State\textbf{Initialization:}
\State$\left(\beta^0, \theta_1^0, \theta_2^0\right)\leftarrow\textsc{random initialization}$
\State$k \leftarrow 0$
\While{$k\leq N$}
\State $t^k \leftarrow [ \theta_1^k ; \theta_2^k ]$\quad\quad\{Forming constituent vector\}
\State $t_1^k\leftarrow\frac{1}{m}\Phi^TX^T(g(X\beta^k) - y)$ 
\State$t_2^k\leftarrow\frac{1}{m}\Psi^TX^T(g(X\beta^k) - y)$
\State$\nabla F^k \leftarrow [ t_1^k ; t_2^k ]$
\quad\quad\{Forming gradient\}
\State${\tilde{t}}^k = t^k - \eta'\nabla F^k$
\quad\{Gradient update\}
\State$[ \theta_1^k ; \theta_2^k ]\leftarrow\mathcal{P}_{s;s}\left(\tilde{t}^k\right)$  
\quad\{Projection\}
\State$\beta^k\leftarrow\Phi \theta_1^k + \Psi \theta_2^k$\quad\{Estimating $\widehat{x}$\}
\State$k\leftarrow k+1$
\EndWhile
\State\textbf{Return:} $\left(\widehat{\theta_1}, \widehat{\theta_2}\right)\leftarrow \left(\theta_1^N, \theta_2^N\right)$
\end{algorithmic}
\end{algorithm}

At a high level, \textsc{STRUCT-DHT} tries to minimize loss function defined in~\eqref{optprob} (tailored to $g$) between the observed samples $y$ and the predicted responses $X\Gamma \widehat{t}$, where $\widehat{t} = [\widehat{\theta}_1; \ \widehat{\theta}_2]$ is the estimate of the parameter vector after $N$ iterations. The algorithm proceeds by iteratively updating the current estimate of $\widehat{t}$ based on a gradient update rule followed by (myopic) \emph{hard thresholding} of the residual onto the set of $s$-sparse vectors in the span of $\Phi$ and $\Psi$. Here, we consider a version of \textsc{DHT}~\cite{SoltaniHegde_Globalsip} which is applicable for the case that coefficient vectors $\theta_1$ and $\theta_2$ have block sparsity. For this setting, we replace the hard thresholding step, $\mathcal{P}_{s;s}$ by component-wise block-hard thresholding~\cite{modelcs}. Specifically, $\mathcal{P}_{s;s}(\tilde{t}^k)$ projects the vector $\tilde{t}^k\in\mathbb{R}^{2p}$ onto the set of concatenated $(s,b)$ block-sparse vectors by projecting the first and the second half of $\tilde{t}^k$ separately.  

Now, we provide our main theorem supporting the convergence analysis and sample complexity (required number of observations for successful estimation of $\theta_1, \theta_2$) of \textsc{STRUCT-DHT}.

\begin{theorem}
\label{mainThConvergence}
Consider the observation model~\eqref{nonlindex} with all the assumption and definitions mentioned in the section~\ref{Perm}. Suppose that the corresponding objective function $F$ satisfies the \new{Structured SRSS/SRSC} properties with constants $M_{6s}$ and $m_{6s}$ 
such that $1\leq\frac{M_{6s}}{m_{6s}}\leq\frac{2}{\sqrt{3}}$ . Choose a step size parameter $\eta'$ with $\frac{0.5}{M_{6s}}<\eta^{\prime}<\frac{1.5}{m_{6s}}$.
Then, \textsc{DHT} outputs a sequence of estimates $(\theta_1^k, \theta_1^k)$ ($t^{k+1} = [\theta_1^k; \theta_1^k]$) such that the estimation error of the parameter vector satisfies the following upper bound (in expectation) for any $k\geq 1$: 
\begin{align}
\label{eq:linconverge}
\|t^{k+1} - \theta\|_2\leq\left(2q\right)^k\|t^0-\theta\|_2 + C\tau\sqrt{\frac{s}{m}}, 
\end{align}
where $q = 2\sqrt{1+{\eta^{\prime}}^2M_{6s}^2-2\eta^{\prime} m_{6s}}$  and $C>0$ is a constant that  depends on the step size $\eta^{\prime}$ and the convergence rate $q$. Here, $\theta$ denotes the true parameter vector defined in section~\ref{Perm}.
\end{theorem}
\begin{proof}[Proof sketch]
The proof follows the technique used to prove Theorem 4.6 in~\cite{soltani2016fast}. The main steps are as follows.
Let $b'\in\mathbb{R}^{2p} =[b_1';b_2']=  t^k - \eta'\nabla F(t^k)$, $b = t^k - \eta'\nabla_J F(t^k)$ where $J = \text{supp}(t^k)\cup \text{supp}(t^{k+1})\cup \text{supp}(\theta)$ and $b_1', \  b_2'\in\mathbb{R}^{p}$ (Here, $\theta = [\theta_1;\theta_2]$ denotes the true parameter vector). Also define $\ t^{k+1} = \mathcal{P}_{s;s}(b') = [\mathcal{P}_s(b'_1); \mathcal{P}_s(b'_2)]$. Now, by the triangle inequality, we have:
$\|t^{k+1} - \theta\|_2\leq \|t^{k+1} - b\|_2 + \|b- \theta\|_2$.
The proof is completed by showing that $\|t^{k+1} - b\|_2\leq 2\|b - \theta\|_2$.
Finally, we use the Khintchine inequality~\cite{vershynin2010introduction} to bound the expectation of the $\ell_2$-norm of the restricted gradient function, $\nabla F(\theta)$ (evaluated at the true parameter vector $\theta$) with respect to the support set $J$).
\end{proof}

\begin{figure}
\begin{center}
\begingroup
\setlength{\tabcolsep}{.1pt} 
\renewcommand{\arraystretch}{.1} 
\begin{tabular}{cc}      
\includegraphics[trim = 8mm 58mm 15mm 30mm, clip, width=0.45\linewidth]{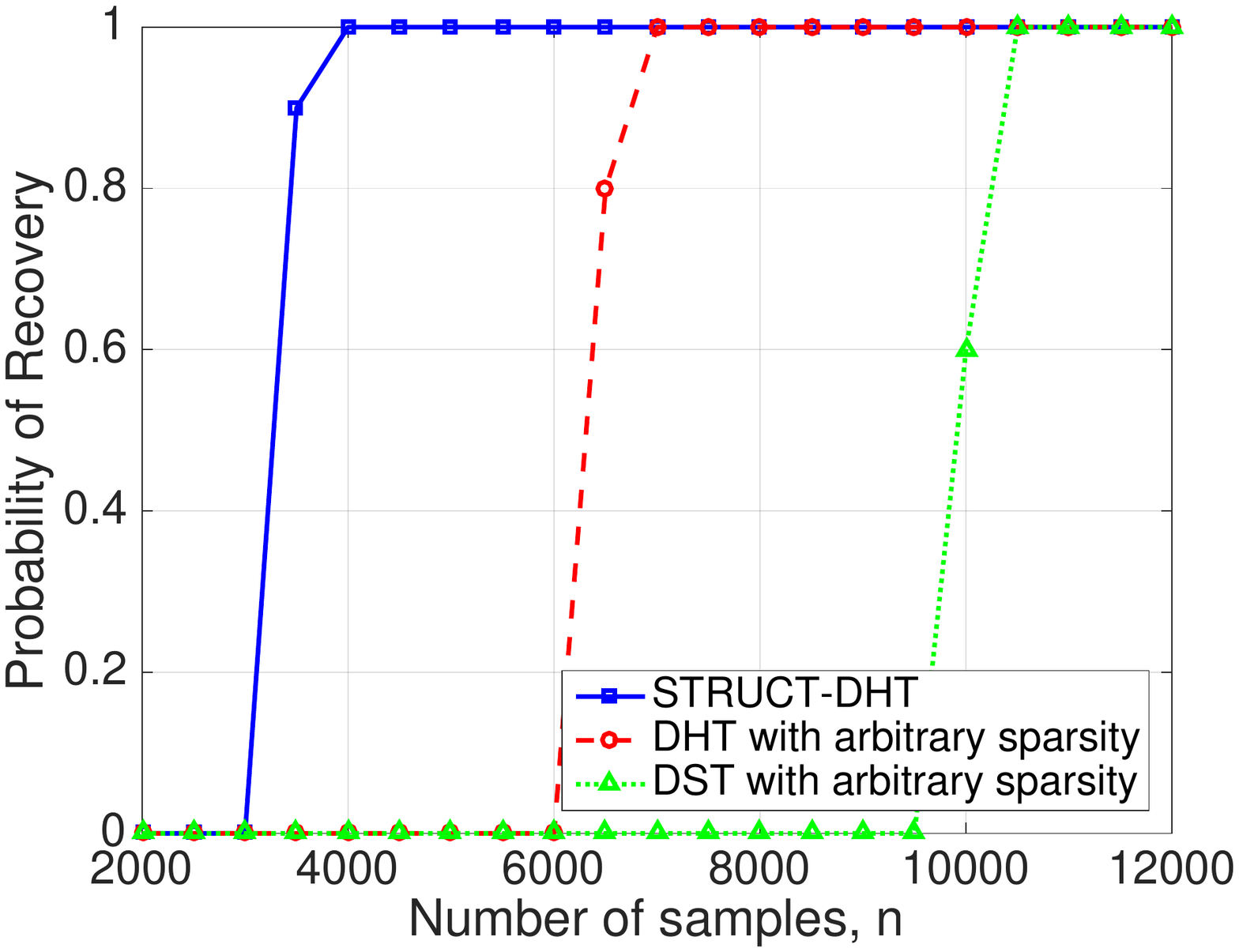}&
\includegraphics[trim = 8mm 58mm 15mm 30mm, clip, width=0.45\linewidth]{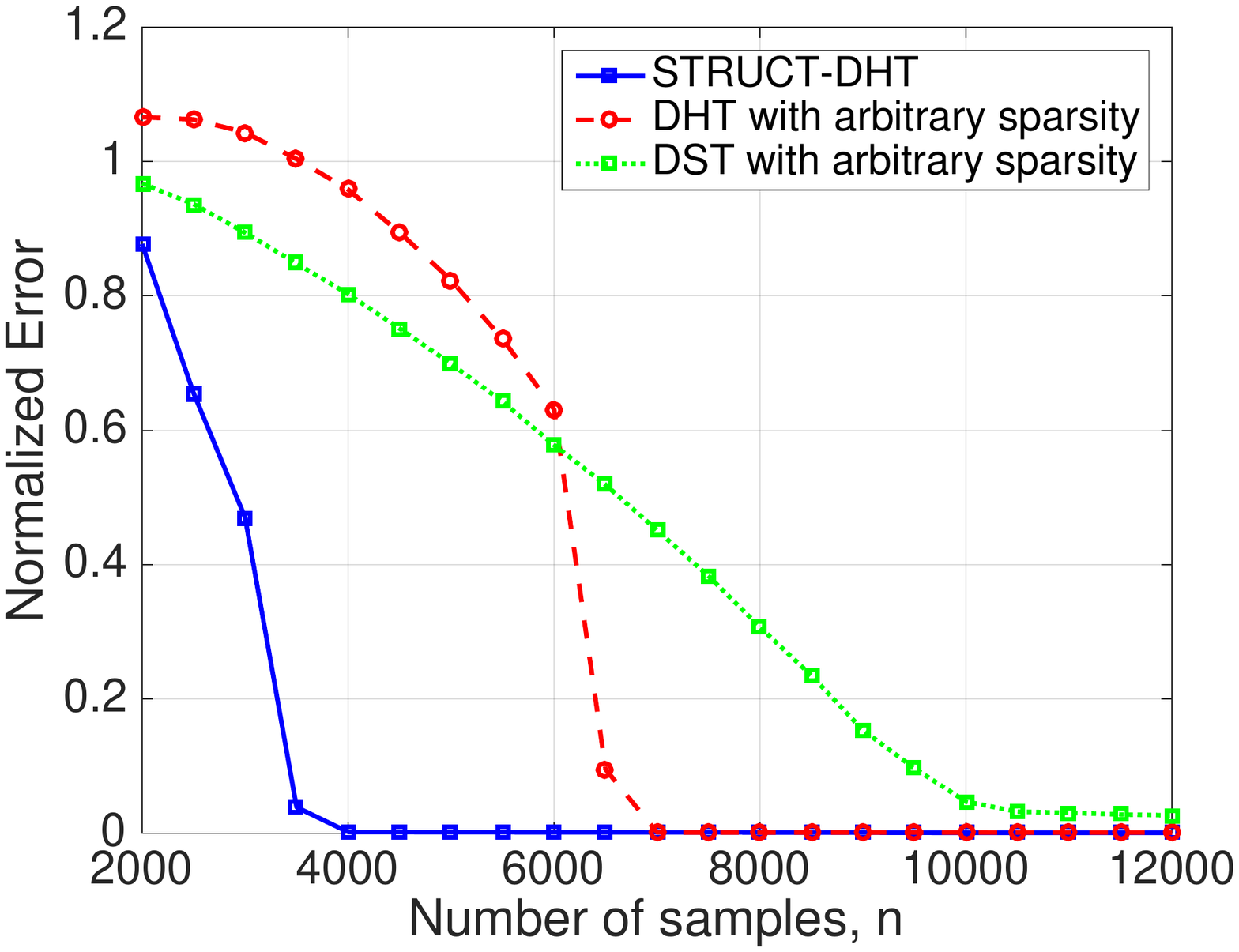}\\
(a) & (b) 
\end{tabular}
\endgroup
\end{center}
\caption{\small{\emph{Comparison of \textsc{DHT} with structured sparsity with other algorithms. (a) Probability of recovery in terms of normalized error. (b) Normalized error between $\widehat{\beta} =  \Phi \widehat{\theta_1} + \Psi \widehat{\theta_2}$ and true $\beta$.}}}
\label{fig:ComparisonSyn}
\end{figure}

Inequality~\eqref{eq:linconverge} indicates the linear convergence behavior of our proposed algorithm. Specifically, in the noiseless scenario to achieve $\kappa$-accuracy in estimating the parameter vector $\widehat{t} = [\widehat{\theta}_1; \ \widehat{\theta}_2]$, \textsc{Struct-DHT} only requires $\log\left(\frac{1}{\kappa}\right)$ iterations.
We also have the following theorem regarding the sample complexity of Alg.\ \ref{algHTM}:
\begin{theorem}
If the rows of $X$ are independent subgaussian random vectors~\cite{vershynin2010introduction}, then the required number of samples for successful estimation of the components, $n$ is given by $\mathcal{O}\left(\frac{s}{b}\log\frac{p}{s}\right)$. Furthermore, if $b = {\Omega}\left(\log\frac{p}{s}\right)$, then the sample complexity of our proposed algorithm is given by $n = \mathcal{O}(s)$, which is asymptotically optimal.
\end{theorem}
\begin{proof}
The proof is similar to the proof of Theorem 4.8 in~\cite{soltani2016fast} where we derived upper bounds on the sample complexity by proving the RSC/RSS for the objective function $F$. Here, the steps are essentially the same as in~\cite{soltani2016fast}, except that we need to compute union bound over the set of $(s,b)$ block-sparse vectors. This set is considerably smaller than the set of \emph{all} sparse vectors and results in an asymptotic gain in sample complexity.  
\end{proof}
The big-Oh constant hides dependencies on various parameters, including the coherence parameter $\varepsilon$, as well as the upper bound and lower bounds on the derivative of the link function $g$.

\section{Numerical results}

To show the efficacy of \textsc{Struct-DHT} for demixing components with structured sparsity, we numerically compare \textsc{Struct-DHT} with ordinary \textsc{DHT} (which does \emph{not} leverage structured sparsity), and also with an adaptation of a convex formulation described in~\cite{yang2015sparse} that we call \emph{Demixing with Soft Thresholding} (\textsc{DST}). We first generate true components $\theta_1$ and $\theta_2$ with length $p = 2^{16}$ with nonzeros grouped in blocks with length $b = 16$ and total sparsity $s = 656$. The nonzero (active) blocks are randomly chosen from a uniform distribution over all possible blocks. 
We construct a design (observation) matrix following the construction of~\cite{krahmer2011new}. Finally, we use a (shifted) sigmoid link function given by $g(x) = \frac{1-e^{-x}}{1 + e^{-x}}$ to generate the observations $y$.  Fig~\ref{fig:ComparisonSyn} shows the the performance of the three algorithms with different number of samples averaged over $10$ Monte Carlo trials. In Fig~\ref{fig:ComparisonSyn}(a), we plot the probability of successful recovery, defined as the fraction of trials where the normalized error is less than 0.05. Fig~\ref{fig:ComparisonSyn}(b) just shows the normalized estimation error for these algorithms. As we can see, \textsc{Struct-DHT} shows much better sample complexity (the required number of samples for obtaining small relative error) as compared to \textsc{DHT} and \textsc{DST}. 

\bibliographystyle{unsrt}
\bibliography{Common/chinbiblio.bib,Common/csbib.bib,Common/mrsbiblio.bib}

\end{document}